\documentclass[11pt,a4paper]{article}

\usepackage[utf8]{inputenc}
\usepackage[T1]{fontenc}
\usepackage{color}
\usepackage[margin=1in]{geometry}
\usepackage{amsmath,amssymb,amsthm,amsfonts}
\usepackage{mathtools}
\usepackage{graphicx}
\usepackage{epstopdf}
\graphicspath{{figures/}}
\usepackage{booktabs}
\usepackage{caption}
\usepackage{subcaption}
\usepackage{placeins}
\usepackage{xcolor}
\usepackage[colorlinks=true,linkcolor=blue,citecolor=blue,urlcolor=blue]{hyperref}
\usepackage[numbers,sort&compress]{natbib}

\newtheorem{theorem}{Theorem}
\newtheorem{proposition}[theorem]{Proposition}

\theoremstyle{definition}

\newtheorem{remark}[theorem]{Remark}

\newcommand{\R}{\mathbb{R}}
\newcommand{\RP}{\mathbb{RP}}
\newcommand{\Sph}{\mathbb{S}}
\newcommand{\norm}[1]{\left\lVert#1\right\rVert}
\newcommand{\PS}{\mathrm{PS}}
\newcommand{\OPS}{\mathrm{OPS}}
\newcommand{\VW}{\mathrm{VW}}
\newcommand{\Sym}{\mathrm{Sym}}
\newcommand{\tr}{\operatorname{tr}}
\newcommand{\edit}[1]{#1}

\title{Multi-View Structure-from-Motion Enables Oriented\\
Projective Shape Analysis in Three Dimensions}

\author{Musab Alamoudi\thanks{Department of Statistics, Florida State University, Tallahassee, FL, USA.}
\and Robert L.\ Paige\thanks{Department of Mathematics and Statistics, Missouri University of Science and Technology, Rolla, MO, USA.}
\and Vic Patrangenaru\thanks{Department of Statistics, Florida State University, Tallahassee, FL, USA.}}

\date{\today}

\begin{document}

\maketitle
\thispagestyle{empty}

\begin{abstract}
Projective shape analysis provides a geometric framework for studying
landmark configurations in digital images acquired by pinhole cameras.
In the classical projective shape (PS) model, three-dimensional
configurations ($k$-ads) are represented as points in $(\RP^3)^q$,
$q=k-5$. A nonparametric test is developed in
\citet{patrangenaru2010}, for this framework, to determine
whether an object matches a design blueprint, with each
configuration reconstructed from a single uncalibrated stereo pair. Such
two-view reconstructions are identified only up to a 3D
projective transformation, which may reverse orientation, so that the
sign-blind PS summary was the only available, before oriented projective
shape (OPS) was recently considered. Multi-view Structure-from-Motion (SfM)
technology removes this obstruction: its bundle adjustment is identified
up to an orientation-preserving projective transformation. In this
paper we revisit a
well-cited three-cube object, from \citet{patrangenaru2010}, with $n=8$
SfM reconstructions built in Agisoft
Metashape Professional~2.3.0, and validate our implementation by
reproducing the published stereo construction from the original data. This
allows us to perform what is to the best of our knowledge the first
three-dimensional OPS analysis, compute its extrinsic
total-variance index and perform statistical inference in this novel
setting. Due to the high concentration of SfM
data, the OPS index is asymptotically one-half the PS index, a structural consequence of concentration rather than a property of
the object. Here our blueprint hypothesis is not rejected
for any of the $q=14$
non-frame landmarks, while the SfM reconstructions are about
$26$ times more concentrated than the stereo ones---a substantial gain
in reconstruction precision. Sample-size, photograph-count, and
frame-ordering analyses support the robustness of these conclusions.

\medskip
\noindent\textbf{Keywords:} projective shape; oriented projective shape;
Structure-from-Motion; Agisoft Metashape; extrinsic mean; pivotal
bootstrap; reproducibility.
\end{abstract}

\clearpage
\setcounter{page}{1}
\section{Introduction}

\subsection{Background and Motivation}

The image of a scene under a pinhole camera is, up to internal
calibration, a central projection \citep{hartleyZisserman2004}.
Statistical shape analysis of three-dimensional configurations recovered
from such images is therefore naturally described with projective
geometry.
\emph{Projective shape} (PS) identifies configurations that differ by a
projective transformation; for $k$ landmarks in general position the
shape space is $(\RP^m)^{k-m-2}$
\citep{mardiaPatrangenaru2005,patrangenaruEllingson2015}. Because the
projective point $[x]\in\RP^m$ identifies the antipodal points $x$
and $-x$ on $\Sph^m$, PS is blind to the front--back orientation of a physical surface.

\citet{patrangenaru2010} provide a nonparametric test of whether the PS of a
manufactured object matches a design blueprint, and illustrated it on a
polyhedral object built from three stacked cubes. This methodology was
subsequently developed in a second installment by \citet{qiu2019}. Here,
each of their eight configurations were
reconstructed from a single uncalibrated \emph{stereo} pair of images.
A fundamental result in computer
vision guarantees that
such a two-view reconstruction recovers the configuration up to a general
projective transformation \citep{faugeras1992,hartley1992}, which is
exactly the ambiguity that PS quotients out. This is what makes PS the
natural framework for stereo data. This paper continues this line of
research by replacing that two-view reconstruction with a multi-view SfM
reconstruction.

More recently, \citet{alamoudi2025} developed \emph{oriented}
projective shape (OPS) analysis, defined on the product of spheres
$(\Sph^m)^q$ rather than $(\RP^m)^q$, which retains the orientation that
PS discards \citep{stolfi1991,choi2022}. OPS requires a \emph{consistent}
sign for
each landmark's unit projective coordinate across the replicates. A single
uncalibrated stereo reconstruction does not provide this, because the
general projective ambiguity can reverse orientation. For this reason
\citet{alamoudi2025} restricted the OPS analysis to the planar case.

\subsection{Aims and Contributions}

This paper has two complementary aims, addressed in the order a reader
would naturally consider them.

\paragraph{Aim 1: Reproduce and Improve PS Analysis with SfM.}
We first ask whether PS analysis of \citet{patrangenaru2010} can be
carried out, and improved, by replacing stereo view with multi-view
Structure-from-Motion (SfM) \citep{schonberger2016}, as implemented in
Agisoft Metashape Professional~2.3.0 \citep{agisoft2025}. We (i) describe
the SfM pipeline and its landmark acquisition; (ii) validate our
implementation by reproducing the published blueprint decision of
\citet{patrangenaru2010} on their own data (Section~\ref{sec:reproduce});
(iii) apply the same PS test to $n=8$ SfM reconstructions of the object;
and (iv) compare the two pipelines on a fair, scale-free footing
(Section~\ref{sec:compare}).

\paragraph{Aim 2: Convert PS to OPS, which SfM Uniquely Enables.}
We then observe that an SfM bundle adjustment is identified only up to a
global \emph{similarity} transformation, which preserves orientation, so
the issues that confined \citet{alamoudi2025} to the plane are no longer
present. We therefore are able to extend the analysis to OPS in three
dimensions (Section~\ref{sec:ops}), provide its extrinsic
total-variance index and inferential methodology, and clarify the precise
relationship between the PS and OPS
indices in the presence of highly concentrated SfM data.

\paragraph{Summary of Findings.}
The blueprint hypothesis is not rejected for any landmark, agreeing with
\citet{patrangenaru2010}. The SfM reconstructions are about $26$ times
more concentrated than the stereo ones and the OPS and PS total-variance
indices coincide up to a factor of two as an algebraic identity. As such
it is seen that the genuine value of OPS is conceptual, in terms of
orientation retention and the detection of mirrored configurations via
the frame-orientation admissibility check, rather than a
difference in the index.
Three sensitivity analyses, based on (i) sample size, (ii) photograph
count and (iii) frame ordering, confirm the robustness of the blueprint
test conclusion.

The paper is organized as follows. Section~2 reviews projective and
oriented projective shape and explains why SfM reconstruction enables
OPS. Section~3 describes the objects involved, the data, and the SfM
pipeline.
Section~4 presents the PS blueprint analysis, its validation on the
2010 data, and the stereo--SfM comparison; Section~5 develops the OPS
analysis and the PS--OPS identity. Section~6 reports robustness and
sensitivity analyses, and Section~7 provides conclusions.

\section{Projective Shape and Reconstruction}
\label{sec:foundations}

\subsection{Projective Frames and Coordinates}

Here points $x\in\R^{m+1}\setminus\{0\}$ are written in homogeneous form, with
the affine embedding $h(u)=[\tilde u]$, $\tilde u=(u^\top,1)^\top$. A
projective frame in $\RP^m$ is an ordered $(m+2)$-tuple of points in
general position. Here $m=3$ and the frame chosen for our work is $F=\{8,12,17,18,19\}$,
which is identical to that of \citet{patrangenaru2010}. Writing
$U_3=[\tilde p_8\mid\tilde p_{12}\mid\tilde p_{17}\mid\tilde p_{18}]$ and
$v_5=U_3^{-1}\tilde p_{19}$, the projective coordinate of a non-frame
landmark $p=[\tilde u]$ is the unit vector
\begin{equation}
\label{eq:proj-coord}
  z(u)=\frac{y(u)}{\norm{y(u)}},\qquad
  y^j(u)=\frac{(U_3^{-1}\tilde u)^j}{v_5^{\,j}},\quad j=1,\dots,M,
\end{equation}
where the superscript $j$ denotes the $j$-th component of a vector.

This coordinate is well defined when $U_3$ is invertible and $v_5$ has
all non-zero entries \citep[Remark~2.1]{patrangenaru2010}. These
coordinates are
invariant under any projective transformation of the configuration, so
they depend on the projective shape alone, not on the camera placement or
the reconstruction scheme.

\subsection{Projective versus Oriented Projective Shape}

In PS the class $[z]\in\RP^m$ identifies $z$ with $-z$; the canonical
embedding is the Veronese--Whitney (VW) map $j_{\VW}([z])=zz^\top$ into
the space $\Sym(M)$ of $M\times M$ symmetric matrices, which is invariant under sign
reversal \citep{bhattacharyaPatrangenaru2005,patrangenaruEllingson2015}.
In OPS the class is the oriented ray
$\overrightarrow{[z]}\in\Sph^m$ and the embedding is the inclusion
$\Sph^m\hookrightarrow\R^{m+1}$ \citep{stolfi1991,choi2022,alamoudi2025},
which
retains orientation. For OPS to be well defined, $z(u)$ must carry a sign
that is consistent across replicates.

\subsection{Why SfM Reconstruction Enables OPS}
\label{sec:why-sfm}

A two-view stereo reconstruction is identified only up to a general
projective transformation \citep{faugeras1992,hartley1992}; since this
group includes orientation-reversing maps, the sign of $z(u)$ may differ
between replicates, and OPS is not well posed. This is the issue
noted by \citet{alamoudi2025}.

Multi-view SfM provides a single global bundle adjustment over all images,
jointly estimating camera intrinsics, camera poses, and 3D landmark
positions \citep{schonberger2016}. The reconstruction is identified only
up to a global \emph{similarity} transformation (rotation, translation,
positive scale). A similarity preserves orientation, so the recovered
point cloud carries a consistent front--back orientation across
replicates. This is exactly the property OPS requires. This is also the
central methodological observation of the paper: \emph{the move from
two-view stereo to multi-view SfM reduces the reconstruction ambiguity
from a general projective transformation to an orientation-preserving
similarity transformation, which is what makes OPS possible in three
dimensions.} Note that projective shape, being invariant under the larger
projective group, is unaffected by this change and remains valid for both
pipelines.

\section{Object, Data, and the SfM Pipeline}
\label{sec:data}

\subsection{Object and the Blueprint}
\label{sec:object}

The object under consideration is the three-cube polyhedron of
\citet{patrangenaru2010}, shown in Figure~\ref{fig:object}. Here, a top
$4\times4\times4$ cube, a middle $6\times6\times6$ cube, and a bottom
$10\times10\times10$ cube, are stacked coaxially so that their vertical
axes coincide. The
\emph{blueprint}, in this setting, is the exact Euclidean design of this
object, meaning the ideal coordinates the manufactured object would have
if it matched the
specification perfectly. It plays the role of the null configuration in
the test (\ref{eq:H-bp}). We ask whether the projective shape recovered
from the photographs is consistent with the projective shape of this
design.

This object has $k=19$ landmarks, which are taken to be its visible
vertices, and
the blueprint assigns each landmark an exact Euclidean coordinate
$(X,Y,Z)$, with $X$ being the measurement on the $x$-axis, $Y$ being the
measurement on the $y$-axis, and $Z$ being the measurement on the
$z$-axis, in the
object's own units. The full blueprint is given in
Table~\ref{tab:blueprint} and is shown, with the projective frame, in
Figure~\ref{fig:object}. The five frame landmarks $\{8,12,17,18,19\}$ are
chosen, as in \citet{patrangenaru2010}, to be in general position; the
remaining $q=14$ landmarks are those whose projective coordinates are
tested against the blueprint. Throughout the paper, the $q=14$ non-frame landmarks keep their original
landmark labels $1$--$7$, $9$--$11$, and $13$--$16$. Thus, landmark
labels in tables and figures extend to $16$, although only $14$ landmarks
are analyzed. This preserves the labeling used in
Figure~\ref{fig:object} and in \citet{patrangenaru2010}.

\begin{table}[htbp]
\centering
\caption{The blueprint: exact Euclidean coordinates $(X,Y,Z)$ of the
$19$ landmarks of the three-cube object, in object units (after
\citet{patrangenaru2010}). Frame
landmarks $\{8,12,17,18,19\}$ are denoted by $^\dagger$. The projective
blueprint coordinates $\gamma_s$ used in the test are obtained by
applying the map (\ref{eq:proj-coord}) to these points
(Remark~\ref{rem:blueprint}).}
\label{tab:blueprint}
\small
\begin{tabular}{rrrr|rrrr}
\toprule
LM & $X$ & $Y$ & $Z$ & LM & $X$ & $Y$ & $Z$\\
\midrule
 1 & 0 & 20 & 0  & 11 & 0 & 10 & 6\\
 2 & 0 & 20 & 4  & 12$^\dagger$ & 6 & 10 & 6\\
 3 & 4 & 20 & 4  & 13 & 6 & 10 & 0\\
 4 & 4 & 20 & 0  & 14 & 0 & 10 & 10\\
 5 & 0 & 16 & 4  & 15 & 10 & 10 & 10\\
 6 & 4 & 16 & 4  & 16 & 10 & 10 & 0\\
 7 & 4 & 16 & 0  & 17$^\dagger$ & 0 & 0 & 10\\
 8$^\dagger$ & 0 & 16 & 6 & 18$^\dagger$ & 10 & 0 & 10\\
 9 & 6 & 16 & 6  & 19$^\dagger$ & 10 & 0 & 0\\
10 & 6 & 16 & 0  &    &   &    &  \\
\bottomrule
\end{tabular}
\end{table}

\begin{figure}[htbp]
\centering
\includegraphics[width=\textwidth]{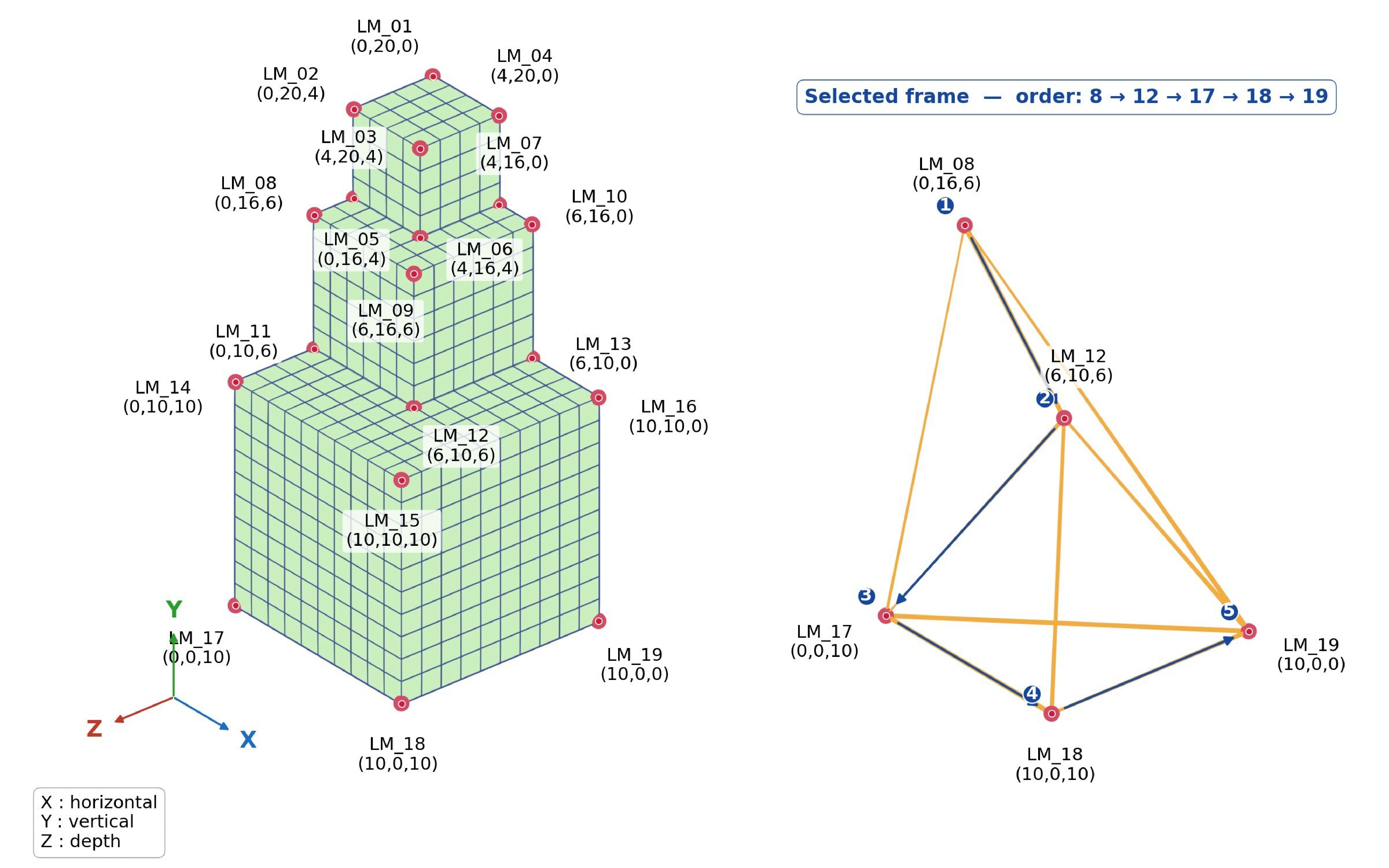}
\caption{\textbf{Left:} the three-cube object with the $19$ landmarks and
the Euclidean blueprint coordinates of Table~\ref{tab:blueprint} ($X$
horizontal, $Y$ vertical, $Z$ depth). \textbf{Right:} the oriented
projective frame $F=\{8,12,17,18,19\}$ in the order used to build $U_3$
and the normalizing point $p_{19}$. Both panels are rendered from the
exact blueprint coordinates of Table~\ref{tab:blueprint} in Blender~5.0
\citep{blender2025}, for legibility; they are not reproduced from
\citet{patrangenaru2010}.}
\label{fig:object}
\end{figure}

\subsection{Acquisition and Reconstruction}

We used Agisoft Metashape Professional~2.3.0 \citep{agisoft2025} to
perform SfM. Our
object was placed on a turntable and photographed with a fixed
camera in manual-focus mode, with the zoom ring fixed. The static
background
was masked and the option \texttt{Exclude stationary tie points} was
enabled, so that tie points on the surrounding scene do not influence
the alignment. Photo alignment was performed at \texttt{High} accuracy
with adaptive
camera-model fitting. Here, the dense point cloud was built at
\texttt{High}
quality with moderate depth filtering. Landmarks were designated as
Agisoft markers and adjusted across several source photographs, so each marker's
3D location was obtained by ray triangulation rather than by projection
onto an interpolated mesh. Figure~\ref{fig:cloud} shows a representative
dense point cloud and the reconstructed model surface.

\begin{figure}[htbp]
\centering
\begin{subfigure}[b]{0.58\textwidth}
  \centering
  \includegraphics[height=4.8cm]{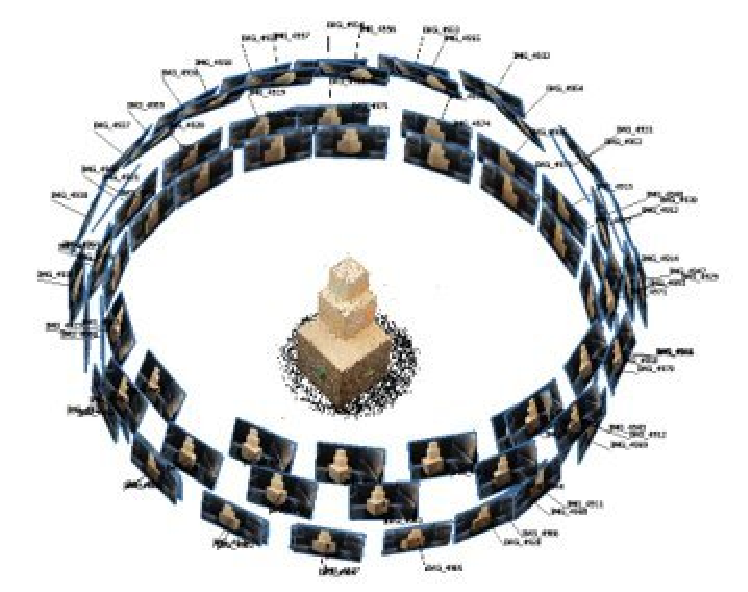}
  \caption{Camera alignment and sparse cloud}
  \label{fig:cloud-align}
\end{subfigure}
\hfill
\begin{subfigure}[b]{0.38\textwidth}
  \centering
  \includegraphics[height=4.8cm]{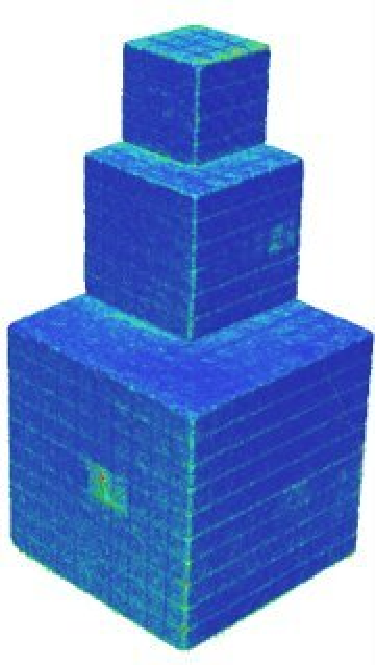}
  \caption{Dense cloud, point confidence}
  \label{fig:cloud-conf}
\end{subfigure}
\caption{Representative SfM output from Agisoft Metashape
Professional~2.3.0. \textbf{(a)} The estimated camera positions
surrounding the object after multi-view photo alignment, with the sparse
tie-point cloud at the center. \textbf{(b)} The dense point cloud of the
object, colored by point confidence (blue high, green lower), used to
assess reconstruction quality before marker placement.}
\label{fig:cloud}
\end{figure}

We acquired $n=8$ independent reconstructions, denoted $M_1,\dots,M_8$,
each from $64$ to $114$ photographs (Table~\ref{tab:models}). The full
marker coordinates are listed in Section~\ref{app:coords}, and the eight
reconstructions are shown in Figure~\ref{fig:scenes}.

\begin{table}[htbp]
\centering
\caption{The eight reconstructions and the number of source photographs
used to build each.}
\label{tab:models}
\small
\begin{tabular}{lcccccccc}
\toprule
Model & $M_1$ & $M_2$ & $M_3$ & $M_4$ & $M_5$ & $M_6$ & $M_7$ & $M_8$\\
Photographs & 64 & 95 & 76 & 109 & 77 & 89 & 84 & 114\\
\bottomrule
\end{tabular}
\end{table}

\begin{figure}[htbp]
\centering
\includegraphics[width=\textwidth]{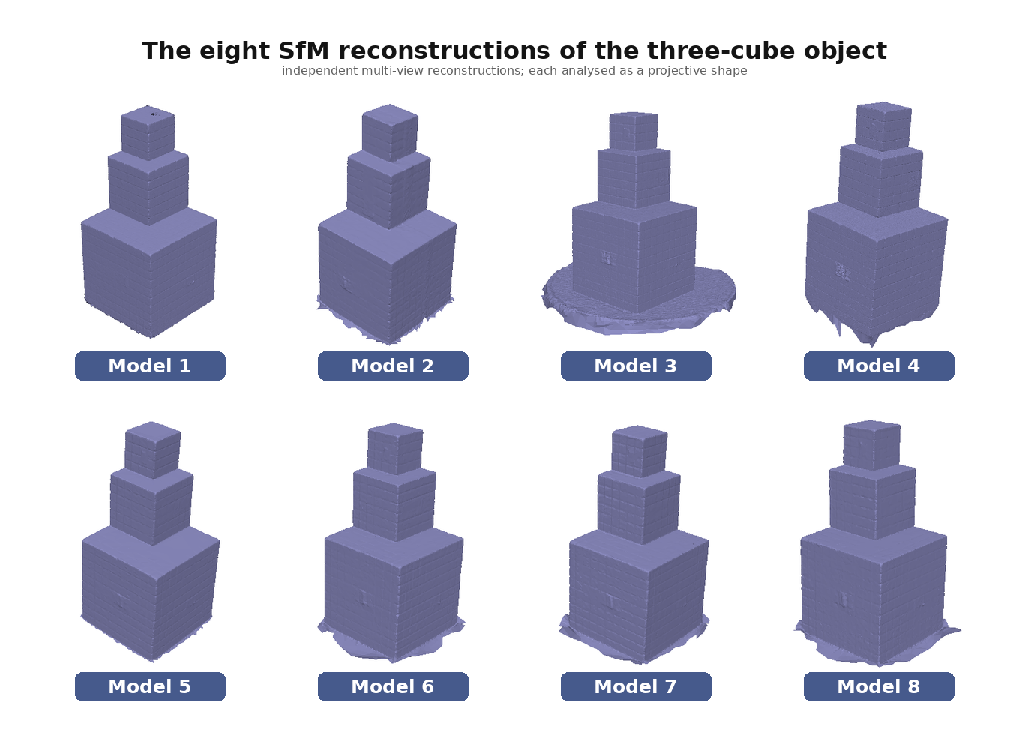}
\caption{The eight independent SfM reconstructions of the three-cube
object, arranged in two rows of four as in
\citet[Fig.~4]{patrangenaru2010}. Each was built from a different
photograph set ($64$ to $114$ images). The panels labelled
``Model~$k$'' correspond to $M_k$ in the text; the bottom face shows the
reconstructed ground plane and minor boundary artifacts of the dense
surface, which do not involve the $19$ landmarks.}
\label{fig:scenes}
\end{figure}

\subsection{Frame Quality and Orientation}

Before proceeding with shape inference we verify that the frame is in
general position by
computing $\det(U_3)$ and the condition number $\kappa(U_3)$, and that
orientation is consistent across reconstructions by checking the sign of
$\det(U_3)$ and of each component of $v_5$. Table~\ref{tab:frame} reports
these diagnostics: all eight reconstructions share the same orientation
($\operatorname{sign}\det U_3=+1$ and an identical sign pattern of
$v_5$) and are numerically well conditioned ($\kappa(U_3)<280$,
$|\det U_3|\ge0.29$).

\begin{table}[htbp]
\centering
\caption{Frame quality and orientation diagnostics. All reconstructions
share orientation and are well conditioned, so both PS and OPS are
well posed.}
\label{tab:frame}
\begin{tabular}{lrrcl}
\toprule
Sample & $\det(U_3)$ & $\kappa(U_3)$ & $\operatorname{sign}(v_5)$ &
Status\\
\midrule
$M_1$ &  0.291 & 122.6 & $(-,+,+,-)$ & OK\\
$M_2$ &  6.730 & 118.3 & $(-,+,+,-)$ & OK\\
$M_3$ &  1.551 &  96.0 & $(-,+,+,-)$ & OK\\
$M_4$ & 32.991 & 229.5 & $(-,+,+,-)$ & OK\\
$M_5$ &  4.187 & 187.4 & $(-,+,+,-)$ & OK\\
$M_6$ &  7.923 & 277.1 & $(-,+,+,-)$ & OK\\
$M_7$ & 10.373 & 161.7 & $(-,+,+,-)$ & OK\\
$M_8$ &  2.688 & 107.2 & $(-,+,+,-)$ & OK\\
\bottomrule
\end{tabular}
\end{table}

\subsection{Caveats for Photogrammetric Pipelines}

Three practical issues affect any photogrammetric pipeline that could be
used for this analysis. First, packages that merge chunks reconstructed from disjoint
photo sets introduce an extra similarity at the merge step; since three
of the five frame landmarks (17, 18, 19) lie on the bottom face of the
largest cube, an inadvertent merge could corrupt the frame. Second,
aggressive surface interpolation causes markers to be placed on a
rendered mesh to inherit interpolation error; interactive placement across source
photographs avoids this. Third, ground-plane or scale-bar calibration may
differ between sessions, but since projective shape is invariant under
any global projective transformation, this part of the pipeline will not
affect our subsequent analysis.

\section{Projective Shape Analysis}
\label{sec:ps}

\subsection{The Blueprint Test}

For landmark $s$, let $X^s_r$ be the sign-blind unit projective
coordinate (\ref{eq:proj-coord}) of reconstruction $r$. The sample VW
second-moment matrix is $J_s=n^{-1}\sum_r X^s_r(X^s_r)^\top$, with
largest eigenvalue $\hat{\lambda}_1^{(s)}$ and eigenvector
$x
)$ (the VW extrinsic sample mean axis). The PS sample extrinsic total-variance
index is $tS_n^{(s),\PS}=2(1-\hat{\lambda}_1^{(s)})$
\citep{bhattacharyaPatrangenaru2005}.

We test, for landmark $s$, whether the mean VW projective shape coordinate equals the
corresponding blueprint coordinate:
\begin{equation}
\label{eq:H-bp}
  H_0^{\text{(bp)}}:\ \mu_s=\gamma_s
  \qquad\text{vs.}\qquad
  H_a^{\text{(bp)}}:\ \mu_s\neq\gamma_s,
\end{equation}
where $\gamma_s$ is the blueprint projective coordinate. Following
\citet[Cor.~4.5]{patrangenaru2010}, the test statistic has a
Hotelling-type quadratic form
\begin{equation}
\label{eq:Ts}
  T_s=n\,\gamma_s^\top D_s\,G_s^{-1}\,D_s^\top\gamma_s,
\end{equation}
where $D_s$ is a tangent basis at $g_s(M)$ and $G_s$ is the $3\times3$
extrinsic sample covariance \citep[eq.~4.13]{patrangenaru2010}, written
$G_{s,n}$ therein. We suppress the sample-size subscript in the present
work. Because
the sample size ($n=8$) is small and equal to that of
\citet{patrangenaru2010} and far below the $n=41$ of
\citet{alamoudi2025}, we use the pivotal bootstrap of
\citet[Cor.~4.5]{patrangenaru2010} rather than a large-sample
chi-square approximation, with $B=20{,}000$ resamples and the
conservative $p$-value $p=(1+\#\{T_b^*\ge T_s\})/(B+1)$. Testing $q=14$
landmarks jointly, we control the familywise error by the Bonferroni
correction at level $\alpha=0.05$, via permutations giving a per-landmark cutoff
$\alpha/q\approx0.00357$, so that $H_0^{\text{(bp)}}$ is rejected when
$p<\alpha/q$ (see \cite{DaHi1997}, \cite{PhSm2010}).

\begin{remark}[Same landmarks are used in the projective frame for the blueprint and for the sample observations]
\label{rem:blueprint}
A projective shape is invariant under a projective transformation of the
configuration; it is \emph{not} invariant under a change of projective
frame. The identification of the shape space with $(\RP^m)^{k-m-2}$ is
made \emph{through} a chosen frame
\citep[\S3.5]{patrangenaruEllingson2015}, so the numerical coordinates,
and every statistic computed from them, depend on which landmarks form
the frame, on the order in which they are assigned to the columns of
$U_3$, on the labeling of the coordinate axes, on the choice of
normalizing point, and on the sign convention. Blueprint coordinates are
therefore comparable with sample coordinates only when both are built the
same way. We accordingly compute $\gamma_s$ from the
Euclidean blueprint using the same frame and conventions applied to the
SfM reconstructions, rather than importing a coordinate table produced
under a different convention. The practical importance of this point is
demonstrated in Section~\ref{sec:reproduce}.
\end{remark}

\subsection{Comparison with results in Patrangenaru et al.\ (2010)}
\label{sec:reproduce}

Before applying the test to new data, we validate our methodology with
the original data from \citet{patrangenaru2010}. Using their eight stereo
reconstructions and their published blueprint projective coordinates,
together with the frame ordering $(17,18,19,12,8)$, our results for the
extrinsic sample mean, were close to theirs with an average angular error of $3.5^\circ$ (with
maximum angular error of
$8.1^\circ$), consistent with the two-decimal rounding of their published
coordinate tables. The selection order $(8,12,17,18,19)$ stated in their
text, or in \citet[Ch.~22]{patrangenaruEllingson2015},
instead gives an average error of $62^\circ$, so the published sample
mean is recovered only under the former ordering. Under it, we obtained the same 
qualitative conclusion as theirs: $H_0^{\text{(bp)}}$ is not rejected for
any of the $14$ landmarks. The statistic $T_s$ agrees with their
published values up to the two-decimal rounding of their coordinate
tables and the sensitivity of covariance matrix inversion in
(\ref{eq:Ts}). The full comparison is given in
Appendix~\ref{app:repro}. This reproduction confirms that the
implementation used for the SfM data is correct, and it concretely
illustrates Remark~\ref{rem:blueprint}, namely that the correct frame
ordering had to
be recovered before the published sample mean could be matched.

\subsection{Blueprint Test on the SfM Data}
\label{sec:ps-results}

When the same test is applied to the $n=8$ SfM reconstructions, the
blueprint
hypothesis is not rejected at the nominal $5\%$ level for $9$ of $14$
landmarks, and the smallest $p$-value over all landmarks is $0.033$
(LM~15), which exceeds the Bonferroni cutoff $0.00357$. After the
Bonferroni correction, $H_0^{\text{(bp)}}$ is therefore not rejected for
any landmark ($0/14$), which is in agreement with
\citet[Sec.~5.1]{patrangenaru2010}. Per-landmark statistics are reported
in Table~\ref{tab:results}.

\begin{table}[htbp]
\centering
\caption{Per-landmark PS and OPS results on the $n=8$ SfM data. The
Bonferroni cutoff for the blueprint test is $\alpha/q=0.00357$; no
$p$-value falls below it. The OPS columns are discussed in
Section~\ref{sec:ops}.}
\label{tab:results}
\small
\begin{tabular}{rrrrrr}
\toprule
LM & $tS_n^{\PS}$ & $tS_n^{\OPS}$ &
   $\widehat{\mathrm{SE}}_{\OPS}$ & $T_s$ & $p$\\
\midrule
 1 & 0.000175 & 0.000087 & 0.000034 &  336.9 & 0.057\\
 2 & 0.001112 & 0.000556 & 0.000250 &  307.4 & 0.052\\
 3 & 0.000759 & 0.000379 & 0.000097 &   57.9 & 0.185\\
 4 & 0.000481 & 0.000241 & 0.000106 &  213.0 & 0.040\\
 5 & 0.000229 & 0.000114 & 0.000060 &  868.6 & 0.054\\
 6 & 0.001412 & 0.000706 & 0.000208 &  440.4 & 0.056\\
 7 & 0.000067 & 0.000033 & 0.000010 &  505.5 & 0.057\\
 9 & 0.000184 & 0.000092 & 0.000028 &  370.4 & 0.078\\
10 & 0.000162 & 0.000081 & 0.000032 &  217.9 & 0.062\\
11 & 0.000503 & 0.000252 & 0.000107 &  287.2 & 0.048\\
13 & 0.000014 & 0.000007 & 0.000002 &  324.1 & 0.037\\
14 & 0.000658 & 0.000329 & 0.000156 &  389.3 & 0.046\\
15 & 0.000064 & 0.000032 & 0.000009 & 1361.5 & 0.033\\
16 & 0.000097 & 0.000049 & 0.000018 &   64.0 & 0.101\\
\bottomrule
\end{tabular}
\end{table}

\subsection{Veronese--Whitney Sample Means and the Blueprint}
\label{sec:vwmeans}

The blueprint test of Section~\ref{sec:ps-results} considers the
dispersion
index and the test decision. In this section, we consider the underlying
Veronese--Whitney (VW) extrinsic sample \emph{means} themselves, since the mean
projective configuration is the natural object to compare against the
blueprint. For each landmark the VW sample mean is the leading eigenvector
$g_s(M)$ of $J_s$, a unit axis in $\RP^3$ reported as a $4$-vector with
its sign aligned to the blueprint for readability. Table~\ref{tab:vwmeans}
lists the $14$ VW sample means together with the leading eigenvalue
$\lambda_1^{(s)}$ and the angular distance of each sample mean from the
blueprint projective coordinate $\gamma_s$.

The VW sample means lie close to the blueprint: the average
angular distance is $2.52^\circ$ (minimum $0.44^\circ$ at landmark~11, maximum $7.19^\circ$
at landmark~6), and every leading eigenvalue exceeds $0.999$, indicating
that the eight reconstructions agree on each sample mean direction to within a
fraction of a degree. Figure~\ref{fig:vwmeans} shows the per-landmark
distances.

\begin{table}[htbp]
\centering
\caption{Veronese--Whitney extrinsic sample means of the $q=14$ non-frame
landmarks ($n=8$ SfM reconstructions). Each sample mean is a unit axis
$g_s(M)=[x,y,z,w]$ in $\RP^3$; $\lambda_1$ is the corresponding leading
eigenvalue of $J_s$ (closer to $1$ indicates tighter agreement); the last
column is the angular distance to the blueprint coordinate $\gamma_s$.}
\label{tab:vwmeans}
\small
\begin{tabular}{rcrr}
\toprule
LM & $g_s(M)=[x,\,y,\,z,\,w]$ & $\lambda_1$ & dist.\ to bp (deg)\\
\midrule
 1 & $[+0.192,+0.546,+0.335,+0.744]$ & $0.99991$ & $2.40$\\
 2 & $[+0.756,-0.339,+0.316,-0.462]$ & $0.99944$ & $4.80$\\
 3 & $[+0.489,-0.220,+0.844,+0.019]$ & $0.99962$ & $2.18$\\
 4 & $[+0.236,+0.663,+0.023,+0.711]$ & $0.99976$ & $4.09$\\
 5 & $[-0.132,+0.509,+0.491,+0.694]$ & $0.99989$ & $2.19$\\
 6 & $[-0.211,+0.762,-0.367,+0.490]$ & $0.99929$ & $7.19$\\
 7 & $[+0.338,+0.594,+0.300,+0.666]$ & $0.99997$ & $1.94$\\
 9 & $[+0.461,-0.000,+0.795,+0.395]$ & $0.99991$ & $0.78$\\
10 & $[+0.371,+0.644,+0.176,+0.645]$ & $0.99992$ & $2.24$\\
11 & $[+0.005,+0.312,+0.849,+0.425]$ & $0.99975$ & $0.44$\\
13 & $[+0.428,+0.542,+0.428,+0.583]$ & $0.99999$ & $0.52$\\
14 & $[+0.692,+0.425,-0.076,+0.579]$ & $0.99967$ & $4.37$\\
15 & $[+0.367,+0.224,+0.636,+0.640]$ & $0.99997$ & $1.16$\\
16 & $[+0.495,+0.622,+0.257,+0.550]$ & $0.99995$ & $1.00$\\
\bottomrule
\end{tabular}
\end{table}

\begin{figure}[htbp]
\centering
\includegraphics[width=0.72\textwidth]{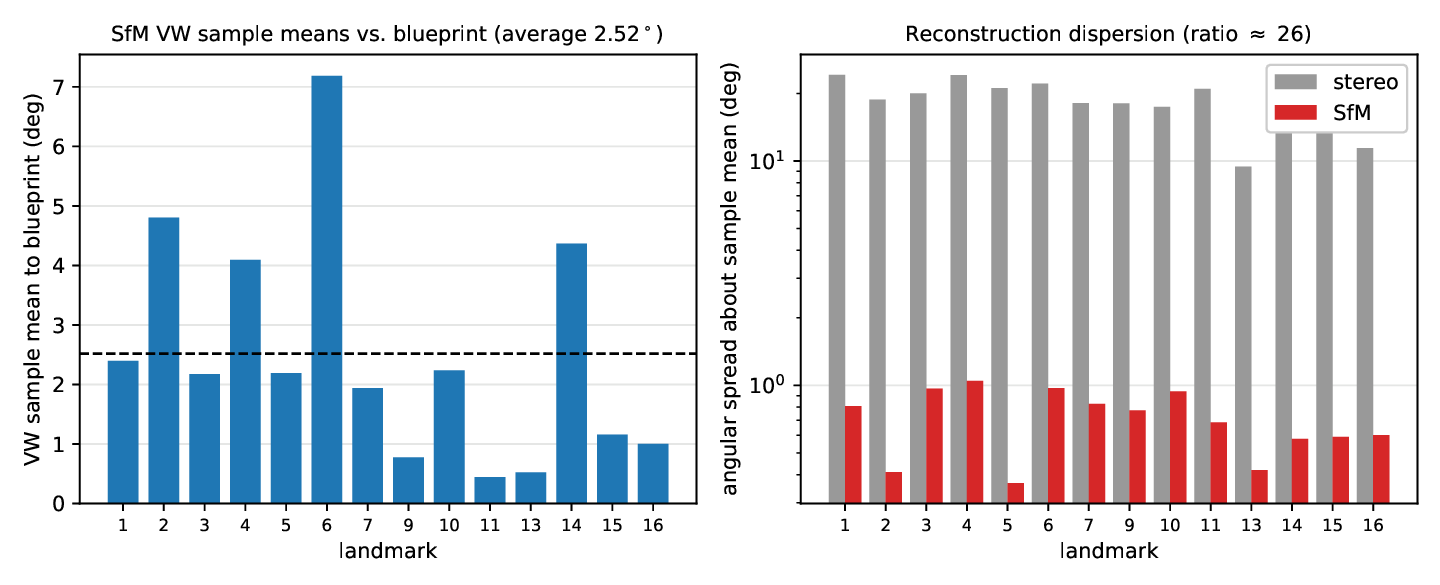}
\caption{\textbf{Left:} angular distance of each VW sample mean from the
blueprint (average angular distance $2.52^\circ$). \textbf{Right:} angular spread of the
eight reconstructions about their VW sample mean, for the multi-view SfM
pipeline (red) and the $2010$ stereo pipeline (gray), under the common
convention of Table~\ref{tab:compare}; the SfM reconstructions are about
$26$ times more concentrated.}
\label{fig:vwmeans}
\end{figure}

Because the SfM reconstruction leaves the global scale arbitrary, the
comparison of a recovered configuration against the blueprint is
meaningful only if the reconstructed cubes retain the blueprint
proportions. Table~\ref{tab:proportions} confirms that they do: the
scale-free edge ratios reproduce the blueprint $4\!:\!6\!:\!10$ to within
$2.4\%$ (middle-to-top) and $4.1\%$ (bottom-to-top), with small
between-reconstruction variation.

\begin{table}[htbp]
\centering
\caption{Scale-free cube-edge ratios of the eight reconstructions against
the blueprint proportions $4\!:\!6\!:\!10$. The reconstructed proportions
match the design, so the VW-sample-mean-versus-blueprint comparison is
well posed.}
\label{tab:proportions}
\small
\begin{tabular}{lcc}
\toprule
Reconstruction & middle/top & bottom/top\\
\midrule
$M_1$ & $1.486$ & $2.558$\\
$M_2$ & $1.474$ & $2.529$\\
$M_3$ & $1.551$ & $2.601$\\
$M_4$ & $1.560$ & $2.593$\\
$M_5$ & $1.554$ & $2.642$\\
$M_6$ & $1.559$ & $2.636$\\
$M_7$ & $1.547$ & $2.622$\\
$M_8$ & $1.555$ & $2.636$\\
\midrule
mean & $1.535$ & $2.602$\\
blueprint & $1.500$ & $2.500$\\
\bottomrule
\end{tabular}
\end{table}

\subsection{Fair Comparison: Stereo versus Multi-View SfM}
\label{sec:compare}

Both pipelines lead to the same blueprint decision. Therefore, an
informative comparison is not the decision but the \emph{precision} of
the recovered
shape, meaning how tightly the eight reconstructions agree. We define our
descriptive
measure of precision to be the average angular spread of each landmark's
projective coordinates about their sample mean direction, computed under
a single coordinate construction applied to both pipelines. We use the
frame
ordering $(17,18,19,12,8)$ under which the published values of
\citet{patrangenaru2010} are reproduced (Section~\ref{sec:reproduce}),
with blueprint-anchored signs. The angular spread
is dimensionless and unaffected by the global scale of the reconstruction.
It is not invariant under the
choice of frame, but applying the same construction to both pipelines
makes the two dispersions directly comparable.
There is a further caveat at the level of individual landmarks. The two
data sources label the landmarks of the same physical object
differently, so a common index selects physically different
corners in the two samples. Aggregate quantities are unaffected, being
computed within each internally consistent sample, but per-landmark
cross-pipeline pairings should be read as comparing dispersion
distributions rather than physically matched corners. This is because
absolute values of
$T_s$ are not comparable
across the two pipelines, because they live in different projective
representations. The angular spread avoids this.

Table~\ref{tab:compare} and Figure~\ref{fig:compare} report this measure.
Using the multi-view SfM pipeline, we recover the projective shape far
more precisely than the original stereo pipeline. The average angular spread
is $0.71^\circ$ against $18.5^\circ$, a reduction by a factor of about
$26$. The sign-blind axis--angle version gives $17.7^\circ$ and a factor
of about $25$, with the same blueprint conclusion. Computing the stereo
spread with raw per-reconstruction signs instead inflates it to
$24.9^\circ$. The excess reflects the orientation inconsistency of
two-view stereo (Section~\ref{sec:why-sfm}) rather than shape
dispersion, which is why the anchored construction is used for the
comparison. On the variance scale these angular factors correspond to
the total-variance ratios of roughly $700$ (PS) and $900$ (OPS). This
precision gain is the practical payoff of replacing two views by
$64$--$114$ views in a single global bundle adjustment.

\begin{table}[htbp]
\centering
\caption{Fair, scale-free comparison of the two acquisition pipelines on
the same object, computed under one construction applied to both
samples: the frame ordering $(17,18,19,12,8)$ under which the published
values of \citet{patrangenaru2010} are reproduced
(Section~\ref{sec:reproduce}) and blueprint-anchored signs.
``Stereo'' is the pipeline of
\citet{patrangenaru2010}; ``SfM'' is the multi-view pipeline used here.
The two sources label the landmarks differently, so the construction is
common at the level of indices, not of physical corners; the aggregate
quantities shown are unaffected. Total-variance indices are
dimensionless; the angular spread is in degrees.}
\label{tab:compare}
\small
\begin{tabular}{lccc}
\toprule
Quantity & Stereo & SfM (this work) & Ratio\\
\midrule
Photographs per reconstruction & $2$ & $64$--$114$ & ---\\
Average angular spread (deg) & $18.5$ & $0.71$ & $\sim\!26\times$\\
Average PS total-variance index & $0.291$ & $4.2\times10^{-4}$ & $\sim\!700\times$\\
Average OPS total-variance index (anchored signs) & $0.188$ & $2.1\times10^{-4}$ &
$\sim\!900\times$\\
Blueprint decision & fail to reject & fail to reject & same\\
\bottomrule
\end{tabular}
\end{table}

\begin{figure}[htbp]
\centering
\includegraphics[width=0.85\textwidth]{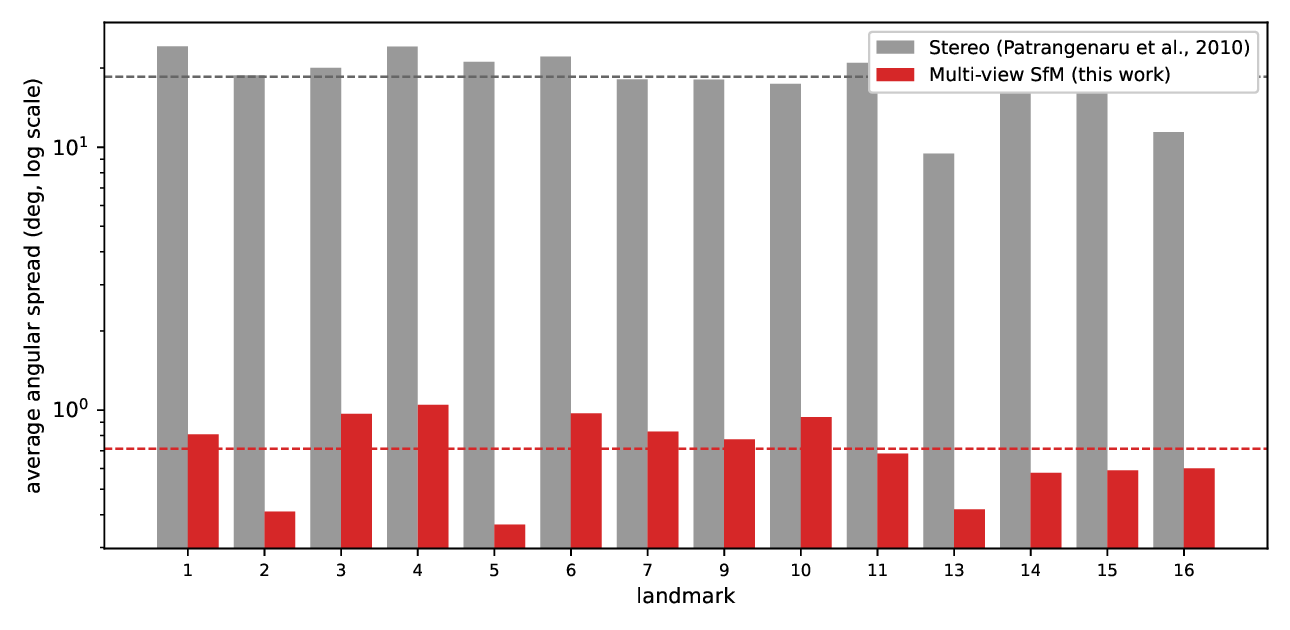}
\caption{Per-landmark reconstruction dispersion (average angular spread
about the sample mean direction, log scale) for the stereo pipeline of
\citet{patrangenaru2010} and the multi-view SfM pipeline used here,
computed under the common convention of Table~\ref{tab:compare}.
Landmark indices refer to each source's own labeling, which differ
between the two samples, so the per-landmark pairing is not a physical
correspondence. The SfM reconstructions are roughly $26$ times more
concentrated.}
\label{fig:compare}
\end{figure}

\section{Oriented Projective Shape Analysis}
\label{sec:ops}

Having established that the SfM reconstructions support the PS analysis,
we now use the property that distinguishes SfM from stereo, namely the
orientation-preserving similarity ambiguity reduction (Section~\ref{sec:why-sfm}),
to carry out oriented projective shape analysis.

\subsection{Oriented Index and Sign Convention}

For landmark $s$, let $Z^s_r$ be the \emph{signed} unit projective
coordinate, with the sign fixed so that $(Z^s_r)^\top\gamma_s\ge0$
for blueprint reference $\gamma_s$. Writing $\bar Z_s=n^{-1}
\sum_r Z^s_r$ and $R_n^{(s)}=\norm{\bar Z_s}$, the OPS extrinsic
total-variance index, introduced in the planar case by
\citet{alamoudi2025}, is
\begin{equation}
\label{eq:tSops}
  tS_n^{(s),\OPS}=2\bigl(1-R_n^{(s)}\bigr).
\end{equation}
A naive rule that makes the largest-magnitude component positive can
reverse
the sign of an entire vector when two components are nearly equal in
magnitude, hence artificially reducing $R_n$ and inflating $tS_n^{\OPS}$.
The
reference-anchored rule is stable under small perturbations because the
inner product with a fixed reference is continuous. For the SfM data
the two rules agree on every sample, so the reported numbers are
unchanged, but the reference-anchored rule is the robust default.

\subsection{Inference for the Oriented Index}

The delta-method standard error of $tS_n^{\OPS}$ is
\citep[eq.~12]{alamoudi2025}
\begin{equation}
\label{eq:SE}
  \widehat{\mathrm{SE}}_{\OPS}^{(s)}
  =\frac{2}{\sqrt n}\,
   \frac{\sqrt{\bar Z_s^\top S_n^{(s)}\bar Z_s}}{\norm{\bar Z_s}},
  \qquad
  S_n^{(s)}=\frac1n\sum_r(Z^s_r-\bar Z_s)(Z^s_r-\bar Z_s)^\top.
\end{equation}
In a test of whether the population OPS total variance is zero,
\begin{equation}
\label{eq:H-var}
  H_0^{\text{(var)}}:\ t\Sigma_{\OPS}^{(s)}=0
  \qquad\text{vs.}\qquad
  H_a^{\text{(var)}}:\ t\Sigma_{\OPS}^{(s)}>0,
\end{equation}
a one-sided alternative is appropriate because a total variance cannot be
negative.
$H_0^{\text{(var)}}$ is rejected when the lower limit of the $95\%$
confidence interval for $t\Sigma_{\OPS}^{(s)}$ exceeds zero. In what
follows, we report
both the delta-method interval from (\ref{eq:SE}) and the percentile
bootstrap interval ($B=20{,}000$).

\subsection{OPS Results}

The per-landmark OPS indices, standard errors, and confidence intervals
appear in Table~\ref{tab:results} and Figures~\ref{fig:bars}
and~\ref{fig:CIs}. The delta-method and bootstrap intervals agree
closely. $H_0^{\text{(var)}}$ is rejected for every landmark by both
intervals. These rejections are expected, since any measurement noise
makes
the population total variance strictly positive. As a result the
informative
quantity is not the binary decision but the \emph{magnitude} of
$tS_n^{\OPS}$. For the SfM data it is of order $10^{-4}$, which is three
orders of magnitude less than the Sope Creek data of
\citet{alamoudi2025}, with a benchmark of $1.87\times10^{-1}$.

\begin{figure}[htbp]
\centering
\includegraphics[width=0.92\textwidth]{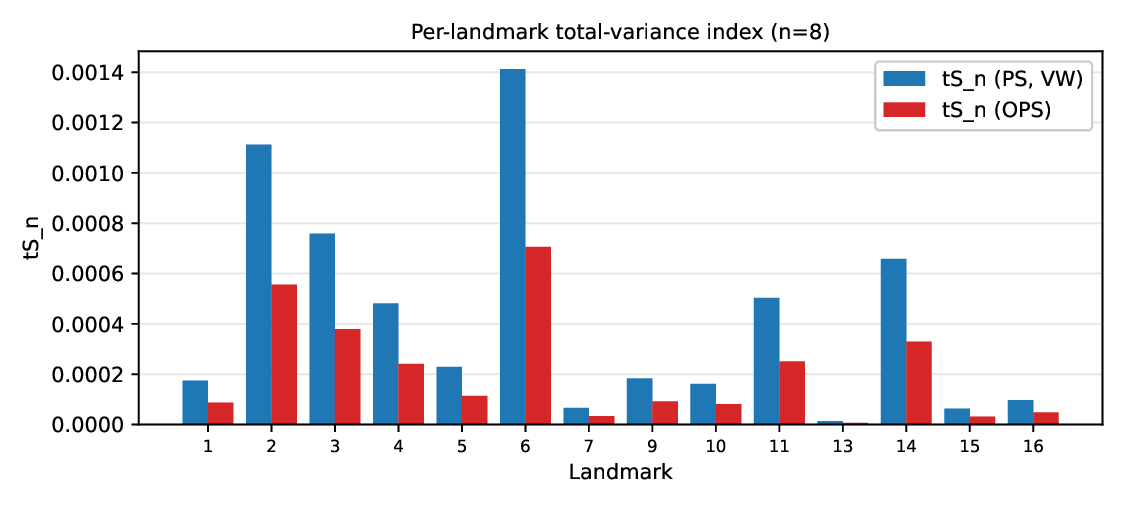}
\caption{Per-landmark PS and OPS total-variance indices, $n=8$. The OPS
index equals one-half of the PS index at every landmark; Section
\ref{sec:identity} shows this is an algebraic consequence of the high
data concentration, not an empirical property of the object.}
\label{fig:bars}
\end{figure}

\begin{figure}[htbp]
\centering
\includegraphics[width=0.92\textwidth]{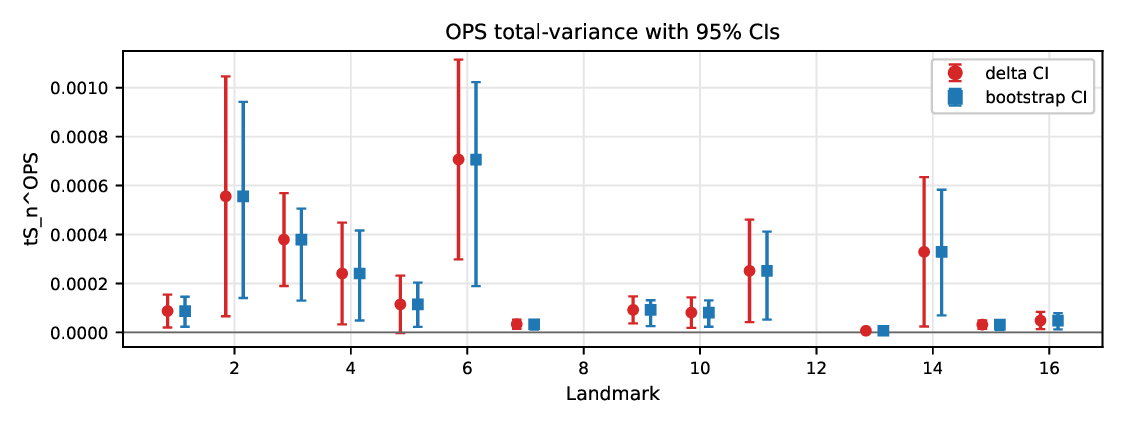}
\caption{OPS total variance with $95\%$ confidence intervals, $n=8$.
Delta-method intervals (circles) and percentile bootstrap intervals
(squares) agree closely. Landmarks~2 and~6 carry the largest dispersion;
landmark~13 the smallest.}
\label{fig:CIs}
\end{figure}

\subsection{The PS--OPS Relationship }
\label{sec:identity}

The average indices over the $q=14$ landmarks are
$\overline{tS_n^{\PS}}=4.2\times10^{-4}$ and
$\overline{tS_n^{\OPS}}=2.1\times10^{-4}$, and the per-landmark ratio
$tS_n^{\OPS}/tS_n^{\PS}$ equals $0.500$ to three decimal places at every
landmark point. We emphasize that this constant is not an empirical
finding
for our particular example but an algebraic identity in the presence of
high-concentration regimes, which we now make precise.

\begin{proposition}
\label{prop:identity}
Let $Z_1,\dots,Z_n\in\Sph^m$ with Euclidean sample mean $\bar Z$ and resultant
$R_n=\norm{\bar Z}$, and let $\lambda_1$ be the largest eigenvalue of
$J=n^{-1}\sum_r Z_rZ_r^\top$. Then $\lambda_1\ge R_n^2$ always, and as
the data concentrate ($R_n\to1$),
\[
  \lambda_1=R_n^2+\edit{O\!\bigl((1-R_n)^2\bigr)},\qquad
  \frac{tS_n^{\OPS}}{tS_n^{\PS}}
  =\frac{2(1-R_n)}{2(1-\lambda_1)}
  \edit{=\frac{1}{1+R_n}+O(1-R_n)}
  \longrightarrow\frac12 .
\]
\end{proposition}

\begin{proof}[Sketch]
Since $\bar Z^\top J\,\bar Z/\norm{\bar Z}^2
=n^{-1}\sum_r(Z_r^\top\bar Z)^2/R_n^2\ge R_n^2$ by Jensen's inequality
applied to $t\mapsto t^2$, the Rayleigh quotient gives
$\lambda_1\ge R_n^2$. Write $Z_r=\bar Z+\varepsilon_r$ with
$\sum_r\varepsilon_r=0$; then
$J=\bar Z\bar Z^\top+E$, \edit{where
$E=n^{-1}\sum_r\varepsilon_r\varepsilon_r^\top\succeq0$
and $\tr E=n^{-1}\sum_r\norm{\varepsilon_r}^2=1-R_n^2$. The first-order
shift of the leading eigenvalue along its eigenvector $\bar Z/R_n$ is
$\bar Z^\top E\,\bar Z/R_n^2
=n^{-1}\sum_r(\bar Z^\top\varepsilon_r)^2/R_n^2
=\operatorname{Var}\!\bigl(\bar Z^\top Z_r\bigr)/R_n^2$, the variance of the
\emph{parallel} components, which is $O\!\bigl((1-R_n)^2\bigr)$; the
second-order (tilt) correction is controlled by the off-diagonal block of
$E$ divided by the spectral gap $R_n^2$ and is likewise
$O\!\bigl((1-R_n)^2\bigr)$. Hence
$\lambda_1=R_n^2+O\!\bigl((1-R_n)^2\bigr)$, so}
$2(1-\lambda_1)=2(1-R_n^2)+O\!\bigl((1-R_n)^2\bigr)
=2(1-R_n)(1+R_n)+O\!\bigl((1-R_n)^2\bigr)$, and dividing by
$2(1-R_n)$ \edit{gives $tS_n^{\OPS}/tS_n^{\PS}=1/(1+R_n)+O(1-R_n)$, which
tends to $\tfrac12$ as $R_n\to1$.}
\end{proof}

On the present data $\lambda_1$ and $R_n^2$ agree to within $10^{-7}$ at
every landmark, so the two indices do not measure dispersion
independently. Therefore, PS and OPS total variance should not be
presented as
separate, mutually corroborating summaries. The genuine value of OPS is
therefore not a smaller index but the conceptual properties in
Table~\ref{tab:ps-ops}. Specifically, that it retains orientation;
that it detects a mirrored configuration---such a configuration fails
the oriented-frame admissibility condition of the OPS construction
\citep{choi2022} and is flagged by the frame-orientation sign
$\operatorname{sign}\det U_3$, the practical certificate under the fixed
affine-lift, similarity-gauge convention used here, while the
$(\Sph^3)^q$ coordinates themselves are reflection-invariant; and that
it requires no covariance inversion
in the test. These advantages are realized in the present study precisely
because the SfM pipeline
supplies the consistent orientation that OPS needs.

\begin{table}[htbp]
\centering
\caption{Conceptual comparison of PS and OPS. The two total-variance
indices coincide up to a factor of two under high concentration
(Proposition~\ref{prop:identity}); the operational differences lie in
orientation and applicability. The flagging of a mirrored configuration
operates through the oriented-frame admissibility condition, certified
in practice by $\operatorname{sign}\det U_3$ under the fixed
affine-lift, similarity-gauge convention used throughout; it is not
carried by the $(\Sph^m)^q$ coordinates, which are reflection-invariant.}
\label{tab:ps-ops}
\small
\begin{tabular}{lll}
\toprule
Property & PS (Veronese--Whitney) & OPS (oriented sphere)\\
\midrule
Shape space & $(\RP^m)^q$ & $(\Sph^m)^q$\\
Identifies $z$ with $-z$ & Yes & No\\
Embedding & $z\mapsto zz^\top$ & $z\hookrightarrow\R^{m+1}$\\
Consistent sign needed & No & Yes (from SfM)\\
Mirrored configuration & Accepted (same shape) & Flagged: inadmissible $k$-ad\\
Covariance inversion in test & Required & Not required\\
Total-variance index & $2(1-\lambda_1)$ & $2(1-R_n)$\\
\bottomrule
\end{tabular}
\end{table}

\FloatBarrier
\section{Robustness and Sensitivity}
\label{sec:robust}

We examine the stability of the blueprint conclusion across different
settings: (i) the set of reconstructions used; (ii) the number of
reconstructions and the number of photographs used in those
reconstructions; and (iii) the frame ordering.

\subsection{Leave-$k$-Out Robustness}

Computing the average OPS index on every leave-one-out and leave-two-out
subsample, in the spirit of \citet[Sec.~4.3]{alamoudi2025}, gives a
baseline value of $2.11\times10^{-4}$. The eight leave-one-out replicates range
over $[1.69,2.31]\times10^{-4}$ (maximum relative change $20\%$), and the
$28$ leave-two-out replicates over $[1.19,2.57]\times10^{-4}$ (maximum
relative change $44\%$). Every replicate stays below $3\times10^{-4}$. As a result the blueprint
conclusion is robust to the choice of reconstructions used.

\subsection{Sample-Size Sensitivity}

Table~\ref{tab:samplesize} and Figure~\ref{fig:samplesize} report the
average OPS index and standard error on the first $n$ reconstructions,
$n=4,\dots,8$. The analysis begins at $n=4$ because the blueprint test
(\ref{eq:Ts}) requires inverting the $3\times3$ extrinsic covariance
$G_s$, which is rank-deficient for $n\le m=3$, where $n=m+1=4$ is the
smallest
sample size at which the test is defined, and for which a single
reconstruction
gives $tS_n^{\OPS}\equiv0$ identically. The index decreases at a
decelerating rate as $n$ increases, and the standard error stabilizes
with $n=7$. The
blueprint hypothesis is never rejected at any sample size, so the failure to
reject the blueprint hypothesis is stable across sample sizes and is not
an artifact of the particular value of $n$.

\begin{table}[htbp]
\centering
\caption{Sample-size sensitivity (cumulative, first $n$ reconstructions).}
\label{tab:samplesize}
\begin{tabular}{rrrc}
\toprule
$n$ & $\overline{tS_n^{\OPS}}$ & $\overline{\mathrm{SE}}_{\OPS}$ &
blueprint reject (Bonf.)\\
\midrule
4 & $2.77\times10^{-4}$ & $1.08\times10^{-4}$ & $0/14$\\
5 & $2.51\times10^{-4}$ & $0.95\times10^{-4}$ & $0/14$\\
6 & $2.26\times10^{-4}$ & $0.84\times10^{-4}$ & $0/14$\\
7 & $2.15\times10^{-4}$ & $0.80\times10^{-4}$ & $0/14$\\
8 & $2.11\times10^{-4}$ & $0.80\times10^{-4}$ & $0/14$\\
\bottomrule
\end{tabular}
\end{table}

\begin{figure}[htbp]
\centering
\includegraphics[width=0.68\textwidth]{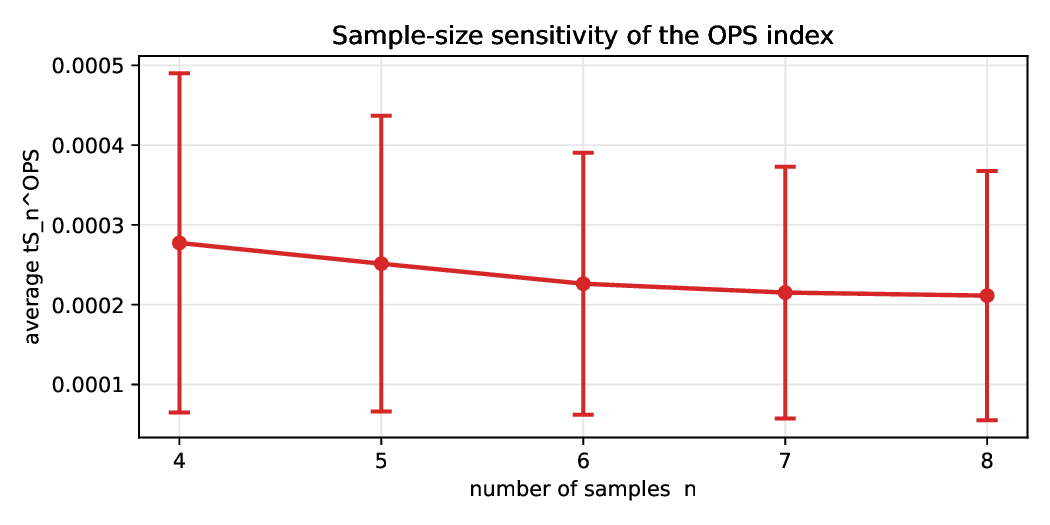}
\caption{Average OPS index with $95\%$ delta-method error bars, computed
cumulatively on the first $n$ reconstructions. The index stabilizes by
$n=7$--$8$.}
\label{fig:samplesize}
\end{figure}

\subsection{Effect of the Number of Photographs}

The eight reconstructions were generated from $64$ to $114$ photographs.
For
each, we computed the average angular deviation of its $14$ projective
coordinates from the blueprint, and correlated it with the photograph
count. The association is moderate and negative (Pearson $r=-0.68$,
$p=0.066$; Spearman $r=-0.55$, $p=0.16$). Here, more photographs tend to
yield
smaller deviations, but the trend does not reach significance at $5\%$
given $n=8$ (Figure~\ref{fig:photo}). Within the range examined,
reconstructions appear adequately stable at $\ge64$ photographs.

\begin{figure}[htbp]
\centering
\includegraphics[width=0.72\textwidth]{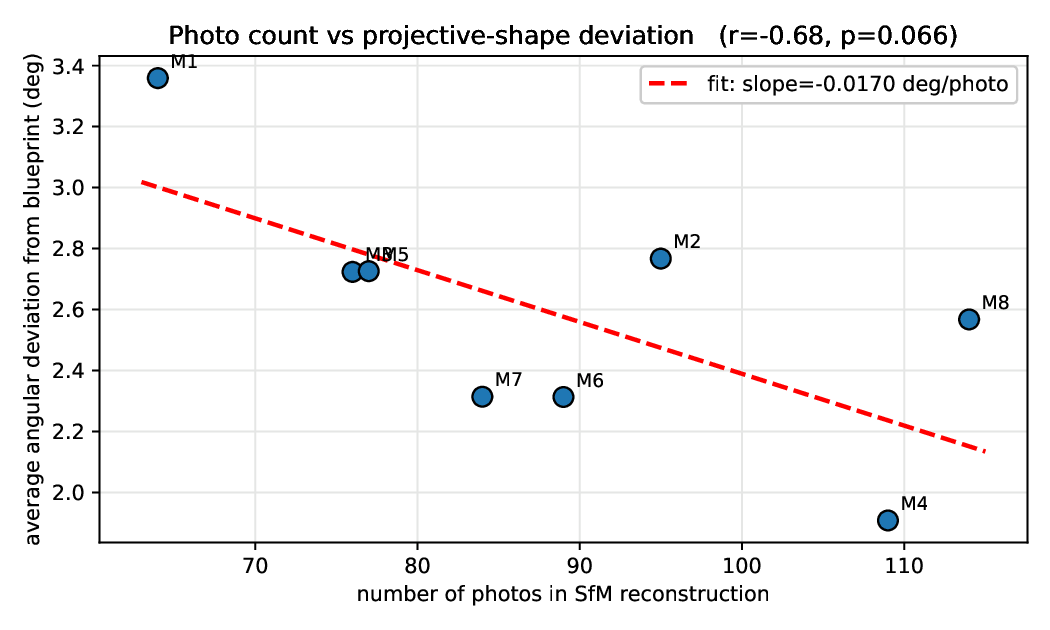}
\caption{Average angular deviation from the blueprint against the number of
photographs used in each reconstruction $M_1,\dots,M_8$. The dashed line
is the least-squares fit (slope $-0.017$ deg/photo; $r=-0.68$,
$p=0.066$).}
\label{fig:photo}
\end{figure}

\subsection{Sensitivity to the Labeling of Landmarks within a Frame}

The projective coordinates of a configuration depend on the projective
frame, including the order in which the five frame landmarks are assigned
to the columns of $U_3$ (Remark~\ref{rem:blueprint}). The blueprint
decision is therefore not invariant under that ordering by construction,
and its stability has to be checked rather than assumed. We checked it by
re-running the test under four orderings of $\{8,12,17,18,19\}$,
including $(17,18,19,12,8)$, the ordering under which the published values
of \citet{patrangenaru2010} are reproduced. Table~\ref{tab:framesens}
shows the Bonferroni
decision is zero for all fourteen landmarks under every ordering, and the
average OPS index varies by
less than $2\%$, so the reported conclusion is not an artifact of the
particular permutation of the frame points. This is an empirical
stability check on one fixed set of frame landmarks, not an invariance
property: Section~\ref{sec:reproduce} shows how strongly the numerical
representation itself can move with the ordering, and a different choice
of frame landmarks would give a different coordinate system again.

\begin{table}[htbp]
\centering
\caption{Sensitivity of the blueprint test to the frame ordering, for the
fixed frame set $\{8,12,17,18,19\}$. The
ordering $(17,18,19,12,8)$ is the one under which the published values
of \citet{patrangenaru2010} are reproduced (Section~\ref{sec:reproduce}).
The
decision is unchanged across these orderings; the numeric representation
is not.}
\label{tab:framesens}
\small
\begin{tabular}{lccc}
\toprule
Frame ordering & blueprint reject (Bonf.) & smallest $p$ &
$\overline{tS_n^{\OPS}}$\\
\midrule
$(8,12,17,18,19)$ (current) & $0/14$ & $0.033$ & $2.11\times10^{-4}$\\
$(17,18,19,12,8)$ (2010 study) & $0/14$ & $0.033$ & $2.08\times10^{-4}$\\
$(19,18,17,12,8)$           & $0/14$ & $0.033$ & $2.08\times10^{-4}$\\
$(12,17,18,19,8)$           & $0/14$ & $0.033$ & $2.08\times10^{-4}$\\
\bottomrule
\end{tabular}
\end{table}

\FloatBarrier
\section{Discussion}
\label{sec:discussion}

Our principal methodological contribution is the finding that
replacing two-view stereo by multi-view SfM upgrades the reconstruction
ambiguity from a general projective transformation to an
orientation-preserving similarity, which is exactly what makes oriented
projective shape applicable in three dimensions without any change to
the underlying statistical theory (Section~\ref{sec:why-sfm}).

Empirically, two points deserve emphasis. First, on the same object and
under the same coordinate construction, the SfM reconstructions are about
$26$ times more concentrated than the original stereo reconstructions
(Section~\ref{sec:compare}), corresponding to total-variance ratios of
roughly $700$ (PS) and $900$ (OPS). The higher precision
does not change the blueprint conclusion, but it makes the recovered
shape far more reproducible. Second, the apparent ``$0.5$ ratio'' between the OPS
and PS indices is an algebraic identity in a high-concentration regime
(Proposition~\ref{prop:identity}), not a property of the object.
Therefore, the
substantive distinction between the frameworks is conceptual
(Table~\ref{tab:ps-ops}).

Why analyze projective shape rather than similarity shape, given that SfM
supplies a metric reconstruction? Because (i) it keeps the analysis
directly
comparable with the projective shape approach of \citet{patrangenaru2010}
and \citet{alamoudi2025}; (ii) projective descriptors remain valid
under any residual projective distortion left by imperfect
self-calibration; and (iii) our aim is to establish that the SfM gauge
makes OPS \emph{definable} in three dimensions. Nonetheless, a
head-to-head
efficiency comparison with similarity shape is a worthwhile but separate
question.

We note three limitations. The sample size ($n=8$) remains
small. With a Bonferroni correction across $14$ landmarks the blueprint
test has limited power, so the failure to reject is best read as being
\emph{compatible with} the blueprint rather than positive
confirmation of it. Note that the smallest unadjusted $p$-value, $0.033$
at
landmark~15, would be significant before correction. The photograph-count
association would likewise benefit from more reconstructions.
The SfM gauge fixes residual scale and orientation choices that are
absorbed by the frame normalization rather than studied directly. And
at this sample size the pivotal bootstrap occasionally produces
near-singular resampled covariances for the PS method. The OPS method,
which does not require covariance inversion, is free of this difficulty
and is in
that sense numerically preferable.

Finally, the agreement of the blueprint test with
\citet{patrangenaru2010}, obtained with a different camera,
reconstruction pipeline, and software, supports the robustness of
projective-shape inference under reasonable changes to the acquisition
step, provided the projective frame is computed consistently
(Remark~\ref{rem:blueprint}). Our implementation reproduces the
fail-to-reject conclusions from \citet{patrangenaru2010} for all $14$
landmarks, based on their own
data (Section~\ref{app:repro}).

\section*{Data and Code Availability}
The reconstructed marker coordinates are listed in Section~\ref{app:coords}.
The Python codes used in this paper can be found at https://socrates.stat.fsu.edu/$\sim$vic/cubes.

\section*{Acknowledgements}
Alamoudi acknowledges support from the Stela Verenca Awarded from the Department of Statistics at Florida State University; Patrangenaru and Paige acknowledge support from awards NSF-DMS:2311059, respectively NSF-DMS:2311058 by the U.S. National Science Foundation .

\section{Reconstructed Marker Coordinates}
\label{app:coords}

Tables~\ref{tab:coords1} and~\ref{tab:coords2} give the full Euclidean
marker coordinates (in Agisoft model units) of all $19$ landmarks for the
eight reconstructions $M_1,\dots,M_8$, exported from Agisoft Metashape
Professional~2.3.0. The SfM reconstruction is defined only up to a global
similarity transformation, so these coordinates carry an arbitrary
overall scale. This
does not affect the projective shape, which is invariant under any global
projective transformation. These are the raw inputs to the
projective-coordinate map (\ref{eq:proj-coord}). The frame landmarks
$\{8,12,17,18,19\}$ are included for completeness.

\begin{table}[htbp]
\centering
\caption{Reconstructed Euclidean marker coordinates (Agisoft model units) for models $M_1$--$M_4$. Landmarks $1$--$19$; the frame is $F=\{8,12,17,18,19\}$.}
\label{tab:coords1}
\tiny
\setlength{\tabcolsep}{3pt}
\begin{tabular}{rrrrrrrrrrrrr}
\toprule
LM & \multicolumn{3}{c}{$M_1$} & \multicolumn{3}{c}{$M_2$} & \multicolumn{3}{c}{$M_3$} & \multicolumn{3}{c}{$M_4$}\\
 & $X$ & $Y$ & $Z$ & $X$ & $Y$ & $Z$ & $X$ & $Y$ & $Z$ & $X$ & $Y$ & $Z$\\
\midrule
1 & 0.095 & 0.843 & -2.677 & 1.139 & 0.410 & -6.804 & 0.879 & 2.436 & -8.575 & -1.244 & 4.063 & -14.478\\
2 & -0.212 & 0.577 & -2.590 & 0.260 & -0.396 & -6.901 & 0.166 & 2.458 & -8.581 & 0.525 & 4.036 & -13.583\\
3 & 0.056 & 0.277 & -2.530 & 1.052 & -1.260 & -6.942 & 0.158 & 2.386 & -7.860 & 1.433 & 4.091 & -15.414\\
4 & 0.359 & 0.539 & -2.612 & 1.931 & -0.466 & -6.839 & 0.878 & 2.364 & -7.853 & -0.368 & 4.125 & -16.295\\
5 & -0.235 & 0.473 & -2.992 & 0.316 & -0.306 & -8.067 & 0.146 & 1.738 & -8.653 & 0.531 & 2.038 & -13.641\\
6 & 0.039 & 0.172 & -2.931 & 1.101 & -1.185 & -8.118 & 0.129 & 1.671 & -7.929 & 1.454 & 2.100 & -15.465\\
7 & 0.339 & 0.431 & -3.016 & 1.984 & -0.378 & -8.014 & 0.845 & 1.651 & -7.935 & -0.344 & 2.149 & -16.340\\
8 & -0.398 & 0.340 & -2.965 & -0.121 & -0.716 & -8.175 & -0.231 & 1.724 & -8.659 & 1.433 & 1.952 & -13.166\\
9 & 0.010 & -0.114 & -2.873 & 1.066 & -2.019 & -8.231 & -0.247 & 1.618 & -7.558 & 2.840 & 2.049 & -15.918\\
10 & 0.480 & 0.280 & -2.997 & 2.394 & -0.801 & -8.077 & 0.848 & 1.590 & -7.555 & 0.114 & 2.087 & -17.280\\
11 & -0.432 & 0.187 & -3.574 & -0.050 & -0.564 & -9.974 & -0.245 & 0.647 & -8.760 & 1.426 & -1.075 & -13.288\\
12 & -0.023 & -0.268 & -3.480 & 1.151 & -1.860 & -10.036 & -0.272 & 0.540 & -7.673 & 2.838 & -0.968 & -15.986\\
13 & 0.449 & 0.125 & -3.605 & 2.462 & -0.654 & -9.857 & 0.815 & 0.509 & -7.654 & 0.128 & -0.922 & -17.399\\
14 & -0.755 & -0.091 & -3.516 & -0.920 & -1.373 & -10.178 & -1.004 & 0.624 & -8.760 & 3.309 & -1.217 & -12.355\\
15 & -0.051 & -0.862 & -3.352 & 1.066 & -3.580 & -10.261 & -1.024 & 0.461 & -6.930 & 5.630 & -1.015 & -16.940\\
16 & 0.735 & -0.195 & -3.566 & 3.295 & -1.574 & -9.951 & 0.813 & 0.395 & -6.910 & 1.060 & -0.994 & -19.279\\
17 & -0.816 & -0.348 & -4.489 & -0.809 & -1.120 & -13.047 & -1.073 & -1.119 & -8.936 & 3.356 & -6.175 & -12.527\\
18 & -0.114 & -1.132 & -4.341 & 1.229 & -3.354 & -13.161 & -1.105 & -1.324 & -7.092 & 5.725 & -6.056 & -17.089\\
19 & 0.681 & -0.456 & -4.552 & 3.481 & -1.322 & -12.876 & 0.727 & -1.397 & -7.072 & 1.176 & -6.042 & -19.416\\
\bottomrule
\end{tabular}
\end{table}

\begin{table}[htbp]
\centering
\caption{Reconstructed Euclidean marker coordinates (Agisoft model units) for models $M_5$--$M_8$.}
\label{tab:coords2}
\tiny
\setlength{\tabcolsep}{3pt}
\begin{tabular}{rrrrrrrrrrrrr}
\toprule
LM & \multicolumn{3}{c}{$M_5$} & \multicolumn{3}{c}{$M_6$} & \multicolumn{3}{c}{$M_7$} & \multicolumn{3}{c}{$M_8$}\\
 & $X$ & $Y$ & $Z$ & $X$ & $Y$ & $Z$ & $X$ & $Y$ & $Z$ & $X$ & $Y$ & $Z$\\
\midrule
1 & -0.567 & 0.978 & -7.074 & -0.467 & 1.907 & -7.641 & -1.726 & 1.108 & -11.060 & 0.413 & 1.879 & -4.576\\
2 & -0.586 & 0.004 & -7.075 & 0.607 & 2.225 & -8.113 & -0.398 & 1.094 & -10.998 & 1.264 & 1.900 & -4.501\\
3 & 0.405 & -0.016 & -7.136 & 0.043 & 2.793 & -9.056 & -0.336 & 1.697 & -12.200 & 1.335 & 2.087 & -5.356\\
4 & 0.428 & 0.972 & -7.149 & -1.059 & 2.469 & -8.605 & -1.681 & 1.699 & -12.278 & 0.468 & 2.056 & -5.437\\
5 & -0.660 & -0.006 & -8.072 & 0.648 & 1.169 & -8.751 & -0.373 & -0.123 & -11.588 & 1.300 & 1.063 & -4.669\\
6 & 0.344 & -0.037 & -8.142 & 0.087 & 1.729 & -9.713 & -0.293 & 0.476 & -12.817 & 1.381 & 1.240 & -5.528\\
7 & 0.352 & 0.955 & -8.133 & -0.999 & 1.407 & -9.231 & -1.643 & 0.499 & -12.866 & 0.521 & 1.217 & -5.591\\
8 & -0.666 & -0.529 & -8.106 & 1.228 & 1.281 & -9.030 & 0.346 & -0.179 & -11.583 & 1.766 & 1.046 & -4.647\\
9 & 0.844 & -0.563 & -8.203 & 0.364 & 2.140 & -10.481 & 0.436 & 0.733 & -13.424 & 1.868 & 1.315 & -5.944\\
10 & 0.868 & 0.941 & -8.210 & -1.299 & 1.662 & -9.753 & -1.608 & 0.756 & -13.516 & 0.559 & 1.274 & -6.053\\
11 & -0.770 & -0.543 & -9.589 & 1.276 & -0.307 & -9.991 & 0.363 & -1.999 & -12.478 & 1.825 & -0.226 & -4.905\\
12 & 0.737 & -0.573 & -9.692 & 0.423 & 0.541 & -11.442 & 0.472 & -1.095 & -14.314 & 1.927 & 0.036 & -6.192\\
13 & 0.766 & 0.939 & -9.713 & -1.255 & 0.043 & -10.731 & -1.591 & -1.083 & -14.423 & 0.610 & -0.014 & -6.308\\
14 & -0.797 & -1.565 & -9.645 & 2.424 & -0.053 & -10.517 & 1.755 & -2.096 & -12.434 & 2.728 & -0.248 & -4.850\\
15 & 1.739 & -1.619 & -9.803 & 0.978 & 1.395 & -12.945 & 1.936 & -0.550 & -15.532 & 2.885 & 0.205 & -7.019\\
16 & 1.799 & 0.910 & -9.843 & -1.850 & 0.565 & -11.764 & -1.512 & -0.530 & -15.729 & 0.682 & 0.122 & -7.209\\
17 & -0.973 & -1.625 & -12.065 & 2.550 & -2.616 & -12.083 & 1.863 & -5.045 & -13.890 & 2.875 & -2.308 & -5.249\\
18 & 1.591 & -1.708 & -12.274 & 1.114 & -1.219 & -14.583 & 2.064 & -3.571 & -17.043 & 3.035 & -1.890 & -7.457\\
19 & 1.652 & 0.867 & -12.311 & -1.751 & -2.053 & -13.388 & -1.423 & -3.547 & -17.238 & 0.806 & -1.991 & -7.654\\
\bottomrule
\end{tabular}
\end{table}

\section{Reproduction of Patrangenaru et al.\ (2010)}
\label{app:repro}

To validate our methodology, we applied it to the original data of
\citet{patrangenaru2010}, in particular the eight reconstructed
configurations of that
study (Table~5 therein) and the blueprint projective coordinates that are
reported
in Table~3 therein, using the frame ordering $(17,18,19,12,8)$, which
recovers the extrinsic sample mean of that study (Table~1 therein) to an
average angular error of $3.5^\circ$ (maximum $8.1^\circ$). Under the
selection order $(8,12,17,18,19)$ stated in the text of that study (and
repeated in \citet[Ch.~22]{patrangenaruEllingson2015}), the
average error is $62^\circ$, so the published values are recovered only
under $(17,18,19,12,8)$. Table~\ref{tab:repro} compares the
blueprint statistic obtained here, $T_s^{\text{here}}$, with the value
reported in that study, $T_s^{2010}$. The qualitative conclusion is
reproduced exactly: $H_0^{\text{(bp)}}$ is not rejected for any of the
$14$ landmarks. The residual numeric differences are consistent with the
two-decimal rounding of the published coordinate tables and with the
sensitivity of the covariance inversion in (\ref{eq:Ts}).

\begin{table}[htbp]
\centering
\caption{Reproduction of the blueprint test of
\citet{patrangenaru2010}. $T_s^{\text{here}}$ is computed by the present
implementation; $T_s^{2010}$ is the published value. Both lead to
``fail to reject $H_0^{\text{(bp)}}$'' for all $14$ landmarks.}
\label{tab:repro}
\small
\begin{tabular}{rcc|rcc}
\toprule
LM & $T_s^{\text{here}}$ & $T_s^{2010}$ & LM & $T_s^{\text{here}}$ &
$T_s^{2010}$\\
\midrule
1 & 4.03 & 3.03 &  9 & 5.89 & 1.57\\
2 & 3.59 & 2.65 & 10 & 5.32 & 5.05\\
3 & 3.35 & 0.15 & 11 & 1.85 & 13.74\\
4 & 5.27 & 3.74 & 13 & 2.20 & 2.04\\
5 & 2.60 & 2.62 & 14 & 3.75 & 3.95\\
6 & 2.30 & 1.79 & 15 & 3.94 & 3.70\\
7 & 2.84 & 3.94 & 16 & 3.15 & 2.87\\
\bottomrule
\end{tabular}
\end{table}

\clearpage
\bibliographystyle{plainnat}

\end{document}